\documentclass{article}

\usepackage[affil-it]{authblk} 

\usepackage{makeidx}         % allows index generation
\usepackage{graphicx}        % standard LaTeX graphics tool
\usepackage[bottom]{footmisc}% places footnotes at page bottom

\usepackage{microtype}
\usepackage{fullpage}
\usepackage{setspace}
\usepackage[T1]{fontenc}

\usepackage{url}
\usepackage{amsthm}
\usepackage{amsmath}
\usepackage{amsfonts}
\usepackage{mathrsfs}
\usepackage{cancel}
\usepackage{amssymb}
\usepackage{enumerate}
\usepackage{color}
\usepackage{silence}
\usepackage{subcaption} % For subfigure
\usepackage{bbm} % For \mathbbm for indicator function
\usepackage{booktabs} % For high-quality tables
\newtheorem{theorem}{Theorem}
\newtheorem{prop}[theorem]{Proposition}

\newtheorem{corollary}[theorem]{Corollary}
\newtheorem{lemma}[theorem]{Lemma}

\newcommand{\vs}{\mathsf{vs}}
\newcommand{\tw}{\mathsf{tw}}
\newcommand{\pw}{\mathsf{pw}}

\newcommand{\MG}{\mathsf{MG}}

\newcommand{\sgn}{\text{sgn}}

\newcommand{\poly}{\mathsf{poly}}

\newcommand{\separ}{\mathsf{S}}

\newcommand{\PG}{\mathsf{PG}}

\begin{document}

\title{Polynomial Threshold Functions of Bounded Tree-Width: Some Explainability and Complexity Aspects}

% \author{Yossi Farjoun%
%   \thanks{Electronic address: \texttt{yfarjoun@math.mit.edu}; Corresponding author}}
% \affil{G. Mill\'an Institute of Fluid Dynamics,\\ Nanoscience and Industrial
% Mathematics,\\ Universidad Carlos III de Madrid, Spain}

\author[1]{Karine Chubarian}
\author[1]{Johnny Joyce}
\author[1,2]{Gy\"orgy Tur\'an}

\affil[1]{University of Illinois at Chicago}
\affil[2]{HUN-REN-SZTE Research Group on AI, Szeged}
\affil[ ]{\emph{\texttt{\{kchuba2,jjoyce22,gyt\}@uic.edu}}}

\maketitle

\abstract{The tree-width of a multivariate polynomial is the tree-width of the hypergraph with hyperedges corresponding to its terms.
Multivariate polynomials of bounded tree-width have been studied by Makowsky and Meer as a new sparsity condition that allows for polynomial solvability of problems which are intractable in general. We consider a variation on this theme for Boolean variables.
A representation of a Boolean function as the sign of a polynomial is called a polynomial threshold representation.
We discuss Boolean functions representable as polynomial threshold functions of bounded tree-width and present two applications to Bayesian network classifiers, a probabilistic graphical model. Both applications are in Explainable Artificial Intelligence (XAI), the research area dealing with the black-box nature of many recent machine learning models.
We also give a separation result between the representational power of positive and general polynomial threshold functions.}

\vspace{1cm}

Dedicated to Janos Makowsky on the occasion of his 75th birthday

\section{Introduction}

The tree-width of a multivariate polynomial is the tree-width of the hypergraph with hyperedges corresponding to its terms.
\emph{Multivariate polynomials of bounded tree-width} have been studied by Makowsky and Meer~\cite{MakMeer,Mako8} as a new sparsity condition that allows for polynomial solvability of problems which are intractable in general.
The paper is closely related to Courcelle, Makowsky and Rotics~\cite{CoMaRo} and to Fischer, Makowsky and Ravve~\cite{Fisch33}. 
The topics covered are \emph{polynomials, graphs, widths, counting} and \emph{logic}. 
We discuss a variation on this theme, in the context of Boolean functions. Polynomials, graphs, widths, counting and (propositional) logic all appear, although the context is different. The problems  considered are from  Explainable Artificial Intelligence (XAI) and complexity theory.

Boolean functions $f : \{0, 1\}^n \to \{0, 1\}$
can be represented by multivariate polynomials exactly or approximately in a many different ways over $GF(2)$ or the reals.
In this paper we consider polynomials over the reals.
As $x^2 = x$ for 0 and 1, polynomials can be assumed to be multilinear. 
Let
\begin{equation} \label{eq:alap}
p(x_1, \ldots, x_n) = \sum_{I \in {\cal I}} \beta_I x_I
\end{equation}
be a multilinear polynomial, where ${\cal I}$ is a family of subsets of $[n]$, $\beta_I \in \mathbb{R}$ and $x_I = \prod_{i \in I} x_i$. 
The degree of the polynomial is the maximal number of variables in a term, and its size is the number of terms. One can consider the Boolean function
\[ \sgn(p(x_1, \ldots, x_n)) \, : \{0, 1\}^n \to \{0, 1\}, \]
where $\sgn$ is the sign function (1 for nonnegative values and 0 otherwise). 
For example,
\begin{equation} \label{eq:ex}
\sgn(x_1 + x_2 + x_3 + x_4 -  x_1 x_2 - x_3 x_4 - 2)
\end{equation}
represents the 4-variable Boolean function which has value 1 iff  $(x_1, x_2, x_3, x_4)$ contains at least three 1s, or both $(x_1, x_2)$ and $(x_3, x_4)$ contain exactly one 1.

Every Boolean function can be represented in this form, called a \emph{polynomial threshold function (PTF) representation} (see Sect.~\ref{sub:boo}).
The smallest degree and size of polynomials representing a Boolean function are important measures of its complexity. 
A Boolean function is a degree-$d$ polynomial threshold function if it can be written as the sign of a degree-$d$ polynomial. Linear threshold functions (LTF or perceptrons) played an important role since the beginning of machine learning.

We consider the \emph{minimal tree-width of polynomial sign-representations} as a measure of complexity of a Boolean functions. 
Consider a polynomial $p$ as in (\ref{eq:alap}).
  The \emph{term-hypergraph} $H_p$ of polynomial $p$ 
has vertex set $[n]$ and edge set ${\cal I}$. 
The tree-width of polynomial $p$ is the tree-width of its term-hypergraph $H_p$. A Boolean function is a tree-width-$k$ polynomial threshold function if it can be written as the sign of a tree-width-$k$ polynomial. Thus, for example, the Boolean function (\ref{eq:ex}) has tree-width 1.

\medskip

Polynomial threshold representations are useful in many different areas. One of these is a probabilistic graphical model, the \emph{Bayesian network classifier (BNC)}~\cite{Friedm97}, a generalization of the Naive Bayes classifier.
A Bayesian network classifier
can be represented by a polynomial threshold function, where the terms reflect the graphical structure of the network (see~Sect.~\ref{sec:bnc_and_ptf}).
%~\cite{Varando15}. 
Polynomials representing the Naive Bayes classifier are linear. \emph{Tree Augmented Bayesian network classifiers (TAN)} have additional edges forming a forest~\cite{Friedm97}~\footnote{\cite{Friedm97} considers trees only, but here one can use forests as well.}. They correspond to quadratic polynomial threshold functions (QTF), where the quadratic terms correspond to the edges of the forest. Bayesian network classifiers with a bounded tree-width network structure correspond to polynomials of bounded tree-width. Bounded tree-width Bayesian networks form a tractable subclass for several inference and learning problems~\cite{Dar34,KolFri}. In this paper we consider some aspects of Bayesian network classifiers related to their explainability.

\medskip

{\bf Explainability}
Recent machine learning models, in particular deep neural networks, are often black boxes in the sense that they do not provide an explanation for the output produced, and this hinders the development of trustworthy AI. Standard examples are that a rejected loan applicant needs to know the reason for the rejection, and a physician needs to know reasons for a suggested diagnosis.
The need for explanations has a long history in machine learning~\footnote{For example, the MYCIN expert system contained an ``Explanation System, which understands simple English questions and answers them in order to justify its decisions or instruct the user''~\cite{Short74}.}, but the recent developments
amplified the relevance of this issue and
brought about Explainable AI as a separate field of research~\cite{molnar2022}.

How to define what is an explanation is a difficult question, which has a long history in the philosophy of science. The type of explanation relevant in a particular machine learning application is context-dependent.
 A \emph{post-hoc} explanation provides an explainable model~\footnote{In the literature an explainable model, such as a decision tree, is usually referred to as an interpretable model. As interpretations are used in logic in a different sense, and also for simplifying terminology, we use explainability for this purpose as well.}
corresponding (exactly or approximately) to a black-box model. Post-hoc explanations can be \textit{global} or \textit{local}, depending on whether the whole model is explained or only its behavior on a specific input.
Even if there is an agreed upon notion of explanation, it is usually not clear how to \emph{evaluate} the explanations provided, one reason being the lack of \emph{ground truth} (e.g., what is the reason for classifying an image as a cat?).

A neural network is usually considered non-explainable and a decision tree is usually considered explainable (but see, for example, \cite{Lipton49} for a discussion of this distinction). Probabilistic graphical models are somewhere in between, having both explainable and non-explainable features. A Bayesian network, on one hand, gives explicit graphical information about conditional independencies. On the other hand, probabilistic inference is hard, and usually not explainable for the user. This holds even for the Naive Bayes classifier. Therefore, explaining Bayesian networks has been studied for a long time~\cite{LacaveD02}.

\medskip

{\bf Explaining Bayesian Network Classifiers by Approximate OBDD} A post-hoc approach to global explanations for Naive Bayes classifiers, proposed by~\cite{ChanD03}, is to compile them into \emph{Ordered Binary Decision Diagrams (OBDD)}~\cite{Wegener00x}. This is an instance of
 \emph{knowledge compilation}, transforming a knowledge representation into another which is more suitable for a given application~\cite{DarMar6}.
Explainability aspects of knowledge compilation are considered 
in~\cite{Colnet98}.

Besides being explainable in the local sense similarly to decision trees, OBDDs can also be considered as global explanations, as they provide a useful data structure for Boolean functions. There are efficient algorithms for deciding their properties and performing operations on them that are relevant for propositional reasoning about a model~\cite{Madre9,Wegener00x,Darwi68}.
Thus, if a system has a component which computes a Boolean function, then having an OBDD representing this function can be useful for reasoning about how this component works as part of the system. In this sense, OBDD is a ``reasonable'' representation.

It was shown in~\cite{HosakaTKY97} that
linear threshold functions require exponential size
OBDDs in the worst case, so such compilation algorithms have to be exponential.
Therefore, for efficient compilation approximations need to be considered. As learned models are supposed to be approximations, this requirement seems natural.

In \cite{Chubarian57} we gave an approximation scheme compiling Bayesian network classifiers of bounded tree-width into OBDD. Bounded tree-width is well-known to be a useful parameterization for efficient computation. Here we consider its relevance for finding explanations efficiently, which is a different aspect.
An outline of the algorithm is given in Sect.~\ref{sec:obd}. From the point of view of knowledge representation in AI, 
the result is an example of \emph{approximate knowledge compilation}. Negative results for approximate knowledge compilation are given in~\cite{Colnet34}.

The algorithm is based on first translating the classifier into a PTF of bounded tree-width, and then translating this PTF
into an approximate OBDD. The second step generalizes the approximate deterministic counting algorithm of~\cite{GopalanKMSVV} (an FPTAS, fully polynomial-time approximation scheme) for the knapsack solution counting problem, corresponding to the LTF case. Error is measured with respect to the input distribution generated by the classifier, and not the uniform distribution. 

Deterministic approximation algorithm schemes for counting the number of satisfying truth assignments for a class of Boolean functions work as follows. Given an $\epsilon > 0$ and a function $f$ from the class, an estimate is computed for the probability $P(f(x) = 1)$, where $x$ is uniformly distributed over $\{0, 1\}^n$. \textit{Multiplicative} (resp., \textit{additive}) approximation algorithms compute an estimate within a multiplicative error $1 \pm \epsilon$ (resp., an additive error $\pm \epsilon$). If the satisfiability problem for the class (i.e., deciding whether there is an $x$ such that $f(x) = 1$) is NP-complete then, assuming $P \neq NP$, polynomial time multiplicative approximation algorithms cannot exist, as those could be used to solve the satisfiability problem in polynomial time. Hence in such cases additive approximation has to be considered.
In contrast to the linear case, for degree at least 2 the PTF satisfiability problem is NP-complete, so only additive approximation algorithms can be hoped for (see~\cite{Serv33}, and~\cite{De33} for the QTF case). Tree-width provides a different parameterization.

\medskip

{\bf Explainability Evaluation for Generalized Additive Models with Interactions} In Sect.~\ref{sec:gam} we discuss another application of polynomial threshold functions to explainability.
A \emph{generalized additive model (GAM)} is a generalization of logistic regression. It can be extended further to 
include interaction terms between the variables. InterpretML~\cite{CarIML} contains an implementation of the $GA^2M$
algorithm of~\cite{Lou13}, which is applied, for example, in the medical domain~\cite{Caruana154}.
Its accuracy is competitive with more complex models, and it is also explainable in the practical, informal sense that the interaction terms are often meaningful for medical experts. 
 
We discuss an experiment aimed at a \emph{qualitative evaluation} of the explainability of the $GA^2M$ algorithm. Bayesian network classifiers provide a class of instances with ground truth.
The applicability of BNC in this context is due to the connection between BNC and polynomial threshold functions mentioned above.
This connection implies that a BNC can be viewed as a generalized additive model with interactions, with the structure of the network corresponding to interaction terms. The ground truth is structured in the sense that terms not included in the basic representation also have a probabilistic meaning in the network. The basic setup of $GA^2M$ is to handle pairwise interactions, therefore so far we have only considered TAN in our experiments.

We present experimental results for a small synthetic example, using InterpretML~\cite{CarIML}. 
BNC is a generative model, so
one can train a $GA^2M$ on random examples generated from the input distribution of the BNC. The ground truth given by the polynomial threshold function
representing the classifier can be compared with the polynomial produced by the $GA^2M$ in different ways, perhaps the simplest being comparing the terms in the ground truth and learned polynomials. Besides achieving almost optimal accuracy on the example, the learned terms show a good correspondence to the ground truth.

\medskip

{\bf Positive Polynomial Threshold Functions} A main research direction of computational complexity theory is to compare the computational power of various models of computation. 

A Boolean circuit computes a Boolean function using $\wedge, \vee$ and $\neg$ gates. The circuit is \emph{monotone} if it contains only $\wedge, \vee$ gates and no negation. 
Such circuits can compute all monotone Boolean functions and only those. Comparing the computational power of monotone and general circuits and other computational models has been studied in complexity theory for a long time~\cite{Pratt75}. The
problem is also referred to as ``monotone versus positive'' (\cite{Ajtai55}, where it is shown that Lyndon's theorem fails for finite models by a circuit lower bound). Superpolynomial separations between monotone and general circuits were given in~\cite{Razb11,Razb22}, strengthened to exponential separation in~\cite{Tardos88}.
A survey of results on monotone circuit complexity is given in~\cite{Jukna147}.

As far as we know the related problems for PTF have not been considered before.
A PTF is \emph{positive} if it is of the form $p(x_1, \ldots, x_n) \ge t$, where $p$ is a polynomial with nonnegative coefficients and no constant term, and
$t$ is nonnegative. Positive PTF represent monotone functions and only those, so one can ask about comparing the representational power of the positive and general versions for monotone Boolean functions.
In Sect.~\ref{sec:mon} we give an example of a QTF of linear size (and tree-width 1) for which every positive QTF has size $\Omega(n^2)$ (and thus tree-width $\Omega(n)$).
There are several open problems in this direction and we mention some of those as well.

\section{Preliminaries} \label{sec:preli}

In this section we give some background on Boolean functions, the notions of width considered and Bayesian network classifiers.

\subsection{Boolean functions} \label{sub:boo}

A Boolean function is of the form $f : \{0, 1\}^n \to \{0, 1\}$. 
For $a = (a_1, \ldots, a_n) \in \{0, 1\}^n$ consider the polynomial
$p_a(x) = \prod_{i=1}^n I_{a_i}(x_i)$, where 
%$I_{a_i(x)$  defined as $I_1(x) = x$ and $I_0(x) = 1 - x.$
for $\ell \in \{0, 1\}$
\begin{equation}
I_\ell(z) = \begin{cases} z & \textrm{if} \,\, \ell = 1 \\ 1 - z & \textrm{if} \,\, \ell = 0 \end{cases}    
\end{equation}
over the reals.
%For $a = (a_1, \ldots, a_n)$ consider the conjunction $\varphi_a = \wedge_{i=1}^n x_i^{a_i}$,
%where $x_i^1 = x_i$ and $x_i^0 = \neq x_i$.
%Let $I_a(x)$ be defined as $I_1(x) = x$ and $I_0(x) = 1 - x.$
Then $p_a(a) = 1$ and $p_a(b) = 0$ for every $b \in \{0, 1\}^n$ different from $a$. Thus $p_f(x) = \sum_{a \, : \, f(a) = 1} p_a(x)$ is is an exact polynomial representation of $f$, which is in fact unique among multilinear polynomials.

The polynomial $p_f(x) - 1/2$ is a polynomial sign-representation of $f$. Thus every Boolean function has a sign-representation of degree at most $n$ and size at most $2^n$. Random Boolean functions require linear degree and exponential size~\cite{Odonn8}.

Truth values can also be represented by $\{\pm 1\}$ instead of $\{0, 1\}$.
In this case $x^2 = 1$ and so polynomials can be assumed to be multilinear as well~\footnote{It is also of interest to consider other truth values, e.g. $\{1, 2\}$~(see \cite{Hansen7,Hajnal73}), and then degrees can matter.}.
The linear mapping $x_i' = 1 - 2 x_i$ for $i = 1, \ldots, n$ provides a transformation between the two representations. 
The unique $\{\pm 1\}$ multilinear representation of a Boolean
function is its Fourier representation~\cite{ODo}.
Comparing both exact and polynomial threshold representations, degrees are preserved, but the number of terms changes. The parity function $x_1 \oplus \ldots \oplus x_n$ (where $\oplus$ is addition in $GF(2)$) requires exponentially many terms over $\{0, 1\}$ even for a polynomial threshold representation~\cite{Krause98}, but it is simply the product over $\{\pm 1\}$.
Being a tree-width-$k$ polynomial threshold function is the same property over $\{0, 1\}$ and $\{\pm 1\}$. The transformations between the two representations do not increase the tree-width of the term-hypergraph, as every new edge is a subset of an already existing edge.

A Boolean function $f \, : \, \{0, 1\}^n \to \{0, 1\}$
is monotone if $x \le y$ implies $f(x) \le f(y)$, where $x = (x_1, \ldots, x_n) \le y = (y_1, \ldots, y_n)$ iff $x_i \le y_i$ for every $i$. 
A PTF $p(x) \ge t$ is \emph{positive} if every coefficient of $p$ is nonnegative, $p$ has no constant term, and $t$ is also nonnegative.
A Boolean function is monotone iff it has a positive PTF representation. 
Every monotone function can be written as a monotone DNF~\cite{Jukna147}; replacing conjunctions with products, disjunctions with sums and using threshold $1/2$ gives a positive PTF. The other direction follows directly from the definitions.

\subsection{Widths}
\label{sec:prel}

For an undirected graph $G = (V, E)$, a \emph{tree-decomposition} of $G$ is given by a tree $T$ and bags $B_t \subseteq V$ for every vertex $t$ of $T$, such that for every $(u, v)\in E$ there exists $t$ such that $u, v\in B_t$ and for every $v \in V$ the vertices $t$ such that $v\in B_t$ form a subtree of $T$. The \emph{width} of a tree decomposition is $\max_{t\in V(T)}|B_t|-1$ and the \emph{tree-width}  $\tw(G)$ of $G$  is the minimal width of tree-decompositions of $G$.  The \emph{path-width}  $\pw(G)$ of $G$ is defined by trees restricted to paths. There are many other notions of width, including those for directed graphs and hypergraphs~\cite{Hlin8}. 

The \emph{moral graph} of a directed acyclic graph (DAG) $G$ is the undirected graph $\MG(G)$ obtained from $G$ by adding undirected edges between co-parents and disregarding the direction of the original directed edges.
The tree-width of $G$ is $\tw(G)=\tw(\MG(G))$ and its path-width is $\pw(G)=\pw(\MG(G))$. This definition is motivated by its use in probabilistic inference in Bayesian networks~\cite{Dar34,KolFri}.

The \emph{primal graph} $\PG(H)$ of a hypergraph $H = (V, E)$ with $E \subseteq P(V)$ is the undirected graph obtained from $H$ by replacing every hyperedge by a clique. The tree-width of $H$ is $\tw(H)=\tw(\PG(H))$ and its path-width is $\pw(H)=\pw(\PG(H))$. 

The \emph{term-hypergraph} $H_p$ of the PTF $p$
has vertex set $[n]$ and edge set ${\cal I}$. The tree-width of a PTF is
the tree-width of its term-hypergraph, and similarly for path-width.

\subsection{Bayesian Network Classifiers} \label{sub:baye}

We use $X_i, x_i$, resp., $a_i$, to denote binary random variables, Boolean variables, resp., Boolean constants. The restriction of a vector $x = (x_1, \ldots, x_n)$ to a subset $I$ of its coordinates is denoted by $x^I$.

A \emph{Bayesian network classifier (BNC)} $N$ is a DAG $G_N$ over binary \emph{input variables} $X_1, \ldots, X_n$ and a binary \emph{classifier variable} $C$, with local conditional probabilities given for each vertex. These specify the conditional probability distribution of the variable corresponding to that vertex, depending on the values of its parents in the DAG. 
It is assumed that $(C, X_i)$ is an edge for every $i = 1, \ldots, n$. The set of additional edges is denoted by $E_N$.
In a \emph{Naive Bayes Classifier} $E_N = \emptyset$.
In a \emph{Tree Augmented Naive Bayes Classifier (TAN)} $E_N$ is an  in-forest, i.e., every node has at most one parent. Fig.~\ref{fig:networkandvalues}
shows a TAN with the specifications of the DAG and the conditional probabilities.

The BNC generates a probability distribution over the random variables 
$C$ and $X_1, \ldots, X_n$ as follows. First a value is generated for $C$ according to its distribution (with no parents in this case), and then values are generated for $X_1, \ldots, X_n$ proceeding down the DAG, always using the conditional probability distribution corresponding to the values generated for the parents. For the TAN of Fig.~\ref{fig:networkandvalues} the random values are generated moving down the tree.

Let $\Pi_i = \{j \, : \, X_j \,\, \textrm{is a parent of} \,\, X_i\}$ be the set of parents of $X_i$ other than $C$.
The \emph{family} of $i$ is $\{i\} \cup \Pi_i$. Let $d_N = 1 + \max_i \left|\Pi_i\right|$ be the maximal size of families in $G_N$, referred to as the \emph{degree} of the Bayesian network. The tree-width of $N$ is $\tw(N)=\tw(\MG(G_N\setminus\{C\}))$, its path-width is  $\pw(N)=\pw(\MG(G_N\setminus\{C\}))$.

 The local conditional probabilities are 
\[
p^0_c = P_N(C = c), \hspace{1cm}
 p^i_{\left(a_i, a_{\Pi_i}, c\right)} = P_N(X_i = a_i \, | \, X_{\Pi_{i}} = a_{\Pi_{i}}, C = c), \]
for $i = 1, \ldots, n$.

The joint distribution of the variables is
\begin{equation} \label{eq:prod3}
P_N(X_1 = a_1, \ldots, X_n = a_n, C = c) = p^0_c \, \prod_{i=1}^n p^i_{\left(a_i, a_{\Pi_i}, c\right)}.
\end{equation}

The marginal distribution over the input variables, also referred to as the \emph{input distribution}, is
\begin{eqnarray*}
P_{N,X}(X_1 = a_1, \ldots, X_n = a_n)
= \sum_{c=0}^1 P_N(X_1 = a_1, \ldots, X_n = a_n, C = c). \end{eqnarray*}

The Bayesian network classifier corresponding to $N$ is a Boolean function $f_N(x_1, \ldots, x_n)$ where
$f_N(a_1, \ldots, a_n) = 1$ iff
\begin{equation}
P_N(a_1, \ldots, a_n, 1) \geq P_N(a_1, \ldots, a_n, 0). \nonumber
\end{equation}

Bayesian network inference and learning problems have been studied in great detail (see \cite{Dar34,KolFri}). In this paper we consider some explainability aspects.

\section{Bayesian Network Classifiers and Polynomial Threshold Functions}\label{sec:bnc_and_ptf}

In this section we formulate the representation of BNC as PTF, which is used in the next two sections.
A unified presentation of work on this connection, going back to~\cite{Min61}, is given in~\cite{Varando15}.

\begin{prop} \label{prop:netp}
Let $N$ be a Bayesian network classifier with non-zero conditional probabilities.
Then there is a polynomial $p$ such that
\[ \log \frac{P_N(C = 1 \, | \, X = x)}{P_N(C = 0 \, | \, X = x)} = p(x), \]
where $p$ is of degree at most $d_N$ and every term is a subset of a family.
Thus $f_N$ is a PTF of degree at most $d_N$.
\end{prop}
\begin{proof} 
It holds that
\begin{eqnarray*}
P_N(X_i = x_i \, | \, X_{\Pi_i} = x_{\Pi_i},  C = c)
=\prod_{\left(a_i, a_{\Pi_i}\right)} {p^i_{\left(a_i, a_{\Pi_i}, c\right)}}^{I_{a_i}(x_i) \prod_{j \in \Pi_i} I_{a_j} (x_j)}.
\end{eqnarray*}
From (\ref{eq:prod3}) we get
\begin{eqnarray*}
P_N(X_1 = x_1, \ldots, X_n = x_n, C = c)
=p^0_c \; \prod_{i= 1}^n \prod_{\left(a_i, a_{\Pi_i}\right)} \left(p^i_{\left(a_i, a_{\Pi_i}, c\right)}\right)^ {I_{a_i}(x_i) \prod_{j \in \Pi_i} I_{a_j} (x_j)}.
\end{eqnarray*}
Then for $x = (x_1, \ldots, x_n)$ it holds that
\[ \frac{P_N(C = 1 \, | \, X = x)}{P_N(C = 0 \, | \, X = x)} = \frac{p^0_1}{p^0_0} \;
\prod_{i= 1}^n \prod_{\left(a_i, a_{\Pi_i}\right)} \left(\frac{p^i_{\left(a_i, a_{\Pi_i}, 1\right)}}
{p^i_{\left(a_i, a_{\Pi_i}, 0\right)}}\right)^ {I_{a_i}(x_i) \prod_{j \in \Pi_i} I_{a_j} (x_j)}.\]
Taking logarithms
\[ \log \frac{P_N(C = 1 \, | \, X = x)}{P_N(C = 0 \, | \, X = x)} = \log \frac{p^0_1}{p^0_0} +
\sum_{i= 1}^n \sum_{(a_i, a_{\Pi_i})} \log \left(\frac{p^i_{\left(a_i, a_{\Pi_i}, 1\right)}}
{p^i_{\left(a_i, a_{\Pi_i}, 0\right)}}\right) \; {I_{a_i}(x_i) \prod_{j \in \Pi_i} I_{a_j} (x_j)}.\]
\end{proof}

\begin{corollary} \label{cor:tanw}
For a TAN with non-zero conditional probabilities $f_N$ is a QTF of tree-width 1, with quadratic terms corresponding to $E_N$.
\end{corollary}

\begin{corollary} \label{cor:width_of_bn_to_width_pft}
For BNC of tree-width $k$ with non-zero conditional probabilities $f_N$
is a PTF of path-width $O(k \log n)$.
\end{corollary}
\begin{proof}
For every edge of the primal graph of the term-hypergraph $H_p$ for the PTF  constructed above there is a family of $N$ containing that edge, which then belongs to the moral graph of the classifier.
The statement then follows from the fact that $\pw(G) = O(\tw(G) \log n)$ for undirected graphs~\cite{Korach98}.
\end{proof}

The primal graph may be a proper subset of the moral graph due to cancellations.

\section{Approximating Bayesian Network Classifiers with OBDD} \label{sec:obd}

In this section we show that Bayesian network classifiers of bounded tree-width can be approximated by polynomial size OBDD. We first introduce OBDD, then formulate the result and outline its proof.

\subsection{Ordered Binary Decision Diagrams}

\begin{figure}[htbp]
    \begin{subfigure}[b]{0.4\textwidth}
        \centering
        \includegraphics[width=\textwidth]{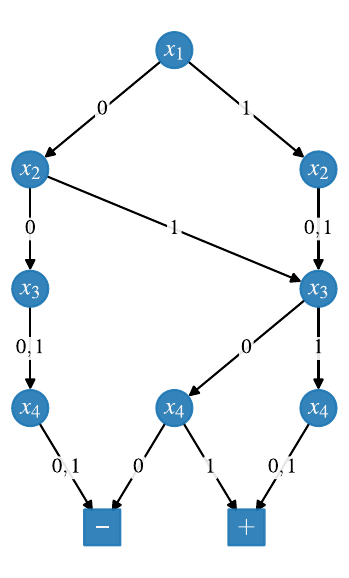}
        \caption{Minimal layered OBDD.}
    \end{subfigure}
    \hfill
    \begin{subfigure}[b]{0.4\textwidth}
        \centering
        \includegraphics[width=\textwidth]{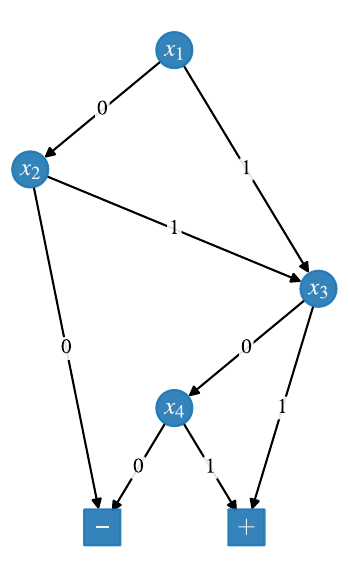}
        \caption{Reduced OBDD.}
    \end{subfigure}
    \caption{OBDDs for function (\ref{eq:ex}). }
    \label{fig:obdd}
\end{figure}

An \emph{ordered binary decision diagram (OBDD)} over Boolean variables $x_1, \ldots, x_n$ computes a Boolean function $f$.
An OBDD
is a DAG with two sinks labeled 0 and 1, and the other nodes labeled with variables. The DAG is assumed here to be layered, with directed edges going from a layer to the next layer, and sinks on the last layer. There are $n + 1$ layers and a permutation $\pi(i)$ of $[n]$ such that nodes on the $i$'th layer are labeled with variable $x_{\pi(i)}$. On the first layer there is a single start node labeled $x_{\pi(1)}$. Every non-sink node has two outgoing edges, labeled with 0, resp., 1. For every truth assignment $a = (a_1, \ldots, a_n)$, $f(a)$ is the label of the sink reached by following edge labels corresponding to the bits in $a$, evaluated in the order given by the labels of the layers. The \emph{width} of an OBDD is the maximal number of nodes in a layer~\footnote{Width is a relevant parameter for the computational power of OBDD: there is a jump at width 5~\cite{Barri33}.}. 

Given an ordering $\pi$ of the variables, every Boolean function has a unique
%\textit{canonical OBDD}, 
minimal layered OBDD, where every node of the $i$'th layer correspond to a subfunction obtained by fixing the variables above the layer.
This can further be simplified to the \textit{reduced} OBDD, where edges can ``jump'' layers, by deleting nodes with both outgoing edges leading to the same successor node, and redirecting their incoming edges to the successor node.
The left side of Fig.~\ref{fig:obdd} shows the minimal layered OBDD of the function (\ref{eq:ex}) and the right side shows its reduced OBDD, with respect to the identity permutation.

A \emph{generator OBDD (GOBDD)} $D$ generates a probability distribution over $\{0, 1\}^n$.
It is similar to a layered OBDD, except edges are also labeled with probabilities, and there is a single sink. A probability $p_u$ is associated with every non-sink vertex $u$, and the 0-edge (resp. 1-edge) leaving $u$ is labeled $p_u^0 = p_u$ (resp., $p_u^1 = 1 - p_u$). For every truth assignment $a = (a_1, \ldots, a_n)$, the GOBDD determines a path from the source to the sink, and $P_D(a)$ is the product of edge probabilities along the path. The \emph{width} of a GOBDD is the width of the underlying OBDD. The \emph{width} of a distribution is the minimal width of GOBDDs generating it. Product distributions have width one.

\subsection{Result and Proof Outline}

We now state the result on approximating Bayesian network classifiers of bounded tree-width with OBDD over the input distribution of the classifier. As discussed in the introduction, the goal is to provide a ``small'' approximate representation of bounded tree-width Bayesian network classifiers which allows for efficient reasoning. Besides the parameters discussed earlier, the bounds also depend on the bit-precision $q$ of the conditional probabilities.

\begin{theorem} \cite{Chubarian57} \label{thm:main_thm}
For every Bayesian network classifier of tree-width $k$ having $n$ Boolean variables and $q$-bit non-zero conditional probabilities,  and every $\varepsilon > 0$ there is an OBDD of size $\poly\left(n^{k}, q, 1/\varepsilon\right)$ approximating the classifier with multiplicative error at most $\varepsilon$ with respect to the input distribution of the classifier.
The OBDD can be constructed in time $\poly\left(n^{k}, q, 1/\varepsilon\right)$.
\end{theorem}

\begin{proof} Let $N$ be a BNC with the properties given in the theorem.
Consider the polynomial $p$ provided by Proposition~\ref{prop:netp} and let $P_{N,X}$ be the input distribution of $N$. We need to construct an approximate OBDD for $sgn(p(x))$, where error is measured with respect to the input distribution.

Consider first the following (not necessarily minimal) layered OBDD for the variable ordering $x_1, \ldots, x_n$, computing $sgn(p(x))$ exactly.
A partial truth assignment $a = (a_1, \ldots, a_\ell)$ to the variables $x_1, \ldots, x_\ell$
will lead to a node labeled by $(s, b)$, where $s$
is the sum of the values of terms of $p$ which contain only variables from $x_1, \ldots, x_\ell$, and 
$b$ is the set of bits from $a$ which occur in so far undetermined terms containing some variables from $x_1, \ldots, x_\ell$. Thus partial truth assignments having the same $(s, b)$ lead to the same node, as those are equivalent for the rest of the computation. 

This OBDD can have exponential size. In order to decrease its size, the following issues need to be taken into consideration:
\begin{enumerate}
\item The variable ordering should be chosen so that $|b|$ is small.

\item Partial sums should be computed approximately by some kind of binning, to get an approximate result with small error.

\item As error is measured with respect to the input distribution,
the binning process needs to have access to the relevant marginals at each point of the computation.
\end{enumerate}

Item 1 is related to tree-width, more precisely to path-width, as the variable ordering is represented by a path. The ordering used is the one given by Corollary~\ref{cor:width_of_bn_to_width_pft} for the primal graph of the term-hypergraph $H_p$.
Item 3 requires an OBDD-like computational representation of the input distribution: the GOBDD defined above. However, it turns out that the GOBDD corresponding to the good path-width ordering of the variables is not necessarily small. This requires another approximation, this time of the input distribution by a small-width GOBDD.

Assume w.l.o.g. that variables are numbered according to the small path-width ordering and let $G = ([n],E)$ be the primal graph. The \emph{separator} of vertex $\ell$ is
\[ \separ_\ell = \{j \, : \, j \le \ell \,\, \text{and} \,\, (j, k) \in E \,\, \text{for some} \,\, k > \ell\}. \]

The \emph{vertex separation number} $\vs(G)$ of $G$ is the minimum of $\max_{\ell\le n-1}|\separ_\ell|$ over all orderings of the vertices. It holds that $\pw(G)=\vs(G)$ (see~\cite{Kinnersley92-uj}) .

\begin{lemma} \label{lem:new-elim}
For any assignment $\{a_1,\dots,a_n,c\}\in\{0,1\}^{n+1}$ and any $\ell \le n$
it holds that
\[P_N\left(a_{\ell}\,|\,a_{1},\dots,a_{\ell-1},c\right)=P_N\left(a_\ell\,|\, a_{\separ_{\ell-1}},c\right).\]
\end{lemma}

Note that the primal graph is undirected, but conditional probabilities are based upon the directed edges of the network, and so the proof of this lemma involves the notion of $d$-separation in Bayesian networks~\cite{Dar34,KolFri}.

Thus the joint distribution can be written as
\[
P_N(X_1=a_1,\dots,X_n=a_n,C=c)=p_{c}^0\prod_{\ell=2}^n P_N\left(a_\ell\,|\, a_{\separ_{\ell-1}},c\right).\]

\begin{lemma}\label{lem:new-joint}
The joint distribution $P_N(X_1,\dots,X_n,C)$ and the conditional distributions $P_N(X_1,\dots,X_n|C)$ have width $n^{O(k)}$.
\end{lemma}

The input distribution $P_{N,X}(a_1, \ldots, a_n)$
can be written as
\[ P_N(a_1, \ldots,a_n \, | \,0) \, P_N(0)+P_N(a_1, \ldots,a_n \, | \,1) \, P_N(1),\]
so by Lemma~\ref{lem:new-joint} it is a mixture of two polynomial-width distributions. However, it is not necessarily of polynomial width itself~\cite{ShenC16}. Nevertheless, it can be approximated by a polynomial-width distribution.
The edge probabilities in a GOBDD for the input distribution can be written as
\[ P_N(a_\ell \, | \,a_1, \ldots,a_{\ell-1})
= \frac{P_N(a_1, \ldots, a_\ell)}{P_N(a_1, \ldots, a_{\ell-1})}
=\frac{\sum_{c=0}^1 P_N(a_1, \ldots, a_\ell, c)}{\sum_{c=0}^1 P_N(a_1, \ldots, a_{\ell - 1}, c)}. \]

The partial truth assignments in the last expression correspond to paths in the GOBDD  of Lemma~\ref{lem:new-joint}, beginning with the start node. This suggests approximating
$P_N(a_\ell \, | \,a_1, \ldots,a_{\ell-1})$ by approximating the terms on the right. This can be achieved by splitting the nodes of the GOBDD, in order to encode an approximation of the conditional probabilities along paths leading to that node. This gives the following approximation of the input distribution.

\begin{lemma}\label{thm:new-delt-app}
There is a distribution $D(X_1, \ldots, X_n)$ of width $\poly\left(n^{k}, q, 1/\varepsilon\right)$ such that for every $a = (a_1, \ldots, a_n)$ it holds that
\[ (1 - \varepsilon) P_D(a) \le P_{N,X}(a) \le (1 + \varepsilon) P_D(a). \]
\end{lemma}

The next step is to construct a ``product'' OBDD computing the PTF \emph{exactly} combining the layered OBDD for the PTF and the small-width approximate GOBBD constructed in the previous lemma. This OBDD has exponential size. It also computes, for every node $v$,  acceptance probabilities of random assignments to the remaining variables, starting at $v$. As every node corresponds to a sum $s$ of the terms already evaluated, these probabilities are monotone functions of $s$.

The last step is to compress the OBDD to polynomial size. 
Nodes of the final OBDD are a set of polynomially many distinguished nodes selected on each level, and other nodes are merged into these nodes. Distinguished nodes are found by binary search based on the $s$ values, using the monotonicity property of the acceptance probabilities to guarantee small error. The compression process is ``virtual'', so the intermediate OBDD does not have to be constructed. \end{proof}

\section{Bayesian Network Classifiers for Evaluating the Explainability of Generalized Additive Models with Interactions}
\label{sec:gam}

\emph{Logistic regression} for Boolean variables is a probabilistic model of the form
\[ \log \frac{P(C = 1 \, | \, X = x)}{P(C = 0 \, | \, X = x)} = \alpha + \sum_{i=1}^n \beta_i x_i,\]
where the coefficients are to be learned.
A \emph{generalized additive model with interactions ($GA^2M$}) has the more general form
\[ g(P(C = 1 \, | \, X = x)) = \sum_{I \in {\cal I}} f_I(x^I),
\]
where $g$ is a given function, ${\cal I}$ is a family of subsets of $[n]$ and the functions $f_I$ are to be learned.

By Proposition~\ref{prop:netp} a BNC can be viewed as \emph{logistic regression model with interactions}. 
In this case the functions $f_I$ are products.
In particular, by Corollary~\ref{cor:tanw}, TAN can be written as

\[ \log \frac{p^0_1}{p^0_0} + \sum_{i \,\, \textrm{\scriptsize root}} \left( \log \frac{p^i_{01}}{p^i_{00}} (1 - x_i)  + \log \frac{p^i_{11}}{p^i_{10}} x_i \right) + \sum_{i \,\, \textrm{\scriptsize non-root}} \Bigg( \left( 
\log \frac{p^i_{001}}{p^i_{000}} (1 - x_i) (1 - x_{f(i)})\right)
\] \[+ \left( 
\log \frac{p^i_{011}}{p^i_{010}} (1 - x_i) x_{f(i)}\right) + \left( 
\log \frac{p^i_{101}}{p^i_{100}} x_i (1 - x_{f(i)})\right) + \left( 
\log \frac{p^i_{111}}{p^i_{110}} x_i x_{f(i)}\right) \Bigg),\]
where the $p^i$s are the local conditional probabilities and $f(i)$ is the parent of $i$.

BNC is a generative model, so
one can generate sets of random training and test examples from the joint distribution starting at the classifier node and proceeding down the forests, choosing nodes according to the local conditional probabilities. Using such samples, the $GA^2M$ algorithm can learn a
logistic regression model $q(x)$ with pairwise interactions.

The accuracy of a (deterministic) classifier $f$ over the whole domain
is

\begin{equation} \label{eq:acc7}
Acc(f) = P_{(X, C) \sim P_N}(C = f(X)) = \sum_x P_N(x, f(x)). 
\end{equation}
By its definition the polynomial $p$ is an \emph{optimal} classifier.  

A $GA^2M$ is considered to be an explainable model by analyzing the polynomial produced. For example, one can consider the set of terms 
${\cal I}$ and the functions $f_I$. In our case this corresponds to considering the terms and their coefficients.
As there is a ground truth, one can consider the \emph{similarity}
of the target polynomial $p$ and the learned polynomial $q$. 
The relevant notion of similarity is to be specified. Here we consider what is perhaps the simplest possibility, the \emph{fraction of common interaction terms}.

\subsection{Experimental Results on a Small Example}\label{subsec:examplebn}

In this section we present a small synthetic example of a target TAN and some experimental results on evaluating its similarity to the learned models. This is a snapshot of ongoing experiments. 

The TAN structure considered is shown in Fig.~\ref{fig:networkandvalues}(a), and the conditional probabilities are listed in Fig.~\ref{fig:networkandvalues}(b). Here, for each node $X_i$ that depends on nodes $X_j$ and $C$, each conditional probability $P(X_i=1\mid X_j=x_j,C=c)$ (for $x_j\in\{0,1\}$ and for $c\in\{0,1\}$) was selected uniformly at random from the union of intervals $(\frac{1}{7},\frac{2}{7})\cup(\frac{3}{7},\frac{4}{7})\cup(\frac{5}{7},\frac{6}{7})$.
For nodes that depend only on $C$, $P(X_i=1\mid C=c)$ was selected in the same way. Finally, we set $P(C=1)=\frac{1}{2}$.

\begin{figure}[htbp]
    \begin{subfigure}[b]{0.48\textwidth}
        \centering
        \includegraphics[width=\textwidth]{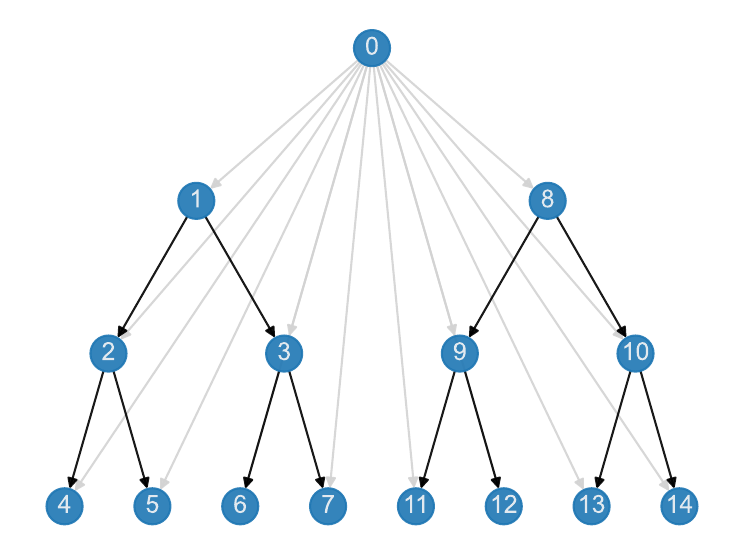}
        \caption{Graph structure. 
        0 is the source node $C$ \newline and $i$ is the node for variable $X_i$.\newline}
    \end{subfigure}
    \begin{subfigure}[b]{0.51\textwidth}
        \centering
        \scalebox{0.75}{
        \centering
        \small
        \setlength\tabcolsep{2.5pt} % Narrower columns
        \renewcommand{\arraystretch}{1.} % Wider rows
        \begin{tabular}{lccccc}
        \toprule
        \textbf{$X$} & Parent & \textbf{$P(1\mid 0, 0)$} & \textbf{$P(1\mid 0, 1)$} & \textbf{$P(1\mid 1, 0)$} & \textbf{$P(1\mid 1, 1)$} \\ 
        \midrule
        \textbf{1}   & - & 0.234 & - & - & 0.732 \\
        \textbf{2}   & 1 & 0.472 & 0.164 & 0.156 & 0.455 \\
        \textbf{3}   & 1 & 0.506 & 0.203 & 0.812 & 0.458 \\
        \textbf{4}   & 2 & 0.810 & 0.202 & 0.508 & 0.734 \\
        \textbf{5}   & 2 & 0.281 & 0.759 & 0.527 & 0.554 \\
        \textbf{6}   & 3 & 0.148 & 0.453 & 0.268 & 0.157 \\
        \textbf{7}   & 3 & 0.743 & 0.470 & 0.449 & 0.540 \\
        \textbf{8}   & - & 0.235 & - & - & 0.479 \\
        \textbf{9}   & 8 & 0.284 & 0.535 & 0.469 & 0.827 \\
        \textbf{10}  & 8 & 0.844 & 0.471 & 0.184 & 0.733 \\
        \textbf{11}  & 9 & 0.559 & 0.471 & 0.798 & 0.224 \\
        \textbf{12}  & 9 & 0.176 & 0.790 & 0.850 & 0.785 \\
        \textbf{13}  & 10 & 0.149 & 0.163 & 0.542 & 0.433 \\
        \textbf{14}  & 10 & 0.483 & 0.220 & 0.236 & 0.194 \\ 
        \bottomrule
        \end{tabular}
        }
        \caption{Conditional probabilities. $P(a\mid b, c)$ is shorthand for $P(X_i=a\mid \text{Parent}(X_i)=b,C=c)$. Values have been rounded to 3dp.}
    \end{subfigure}
    \caption{The TAN used in the experiment. }
    \label{fig:networkandvalues}
\end{figure}

We use the InterpretML software package~\cite{CarIML} Python implementation of $GA^2M$, referred to as explainable boosting machines (EBM)~\footnote{$GA^2M$ returns functions $f_i, f_{ij}$ as step functions and are rewritten in polynomial form.}. The number of interaction terms is specified to be 12, the number of edges between the $X$-variables. All other settings were left as default. To ensure accurate arithmetic with the TAN, particularly with repeated multiplication, all operations were performed with Python's decimal module. Pgmpy \cite{pgmpy} was used to build the graph structure of the TAN, and NetworkX \cite{NetworkX} was used for visualization.
Classifiers were trained on samples of sizes up to 2000 (out of $2^{14}$ possible inputs), with step size 5, using 5-fold cross validation. Accuracies are shown by the dashed curve.

Accuracies over the whole domain (as in (\ref{eq:acc7})) are also shown on the top part of Fig.~\ref{fig:accuracyvssize}. The accuracy of the target classifier $p(x)$ is 0.9266, corresponding the top line. The accuracies of the learned classifiers $q(x)$ (averaged over the 5 folds) are shown by the thick middle curve.

The explanation quality of the classifiers produced 
is measured by the fraction of interaction terms which are common to $p$ and $q$, for every sample size averaged over the 5 folds. The percentages are shown by the set of points in the lower half of Fig.~\ref{fig:accuracyvssize}. The percentages are growing with sample size and stabilize around $70 \%$.

\begin{figure}[hbt]
    \centering
    \includegraphics[width=0.8\textwidth]{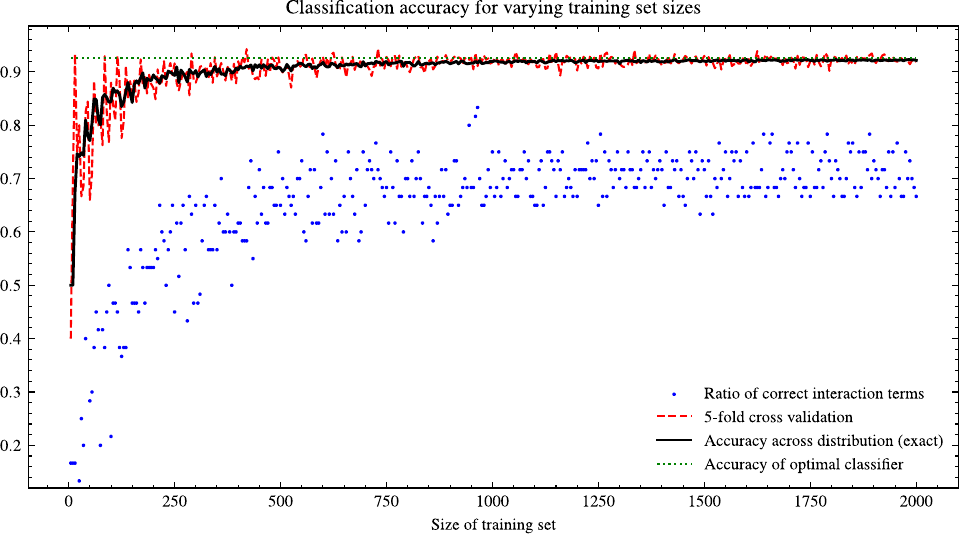}
    \caption{Optimal classifier versus learned classifiers: accuracy and term overlap}
    \label{fig:accuracyvssize}
\end{figure}

In order to get more detailed information, we present the coefficients of the optimal classifier $p$ and three classifiers $q_1, q_2, q_3$ obtained from sample size 1600 in  Tables~\ref{table:polytermslinear} and \ref{table:polytermsquadratic} showing the affine, resp., quadratic terms. The accuracy of the classifiers is 0.9215, 0.9212 and 0.9198, respectively.

\begin{table}[htbp]\caption{Coefficients of \emph{affine} terms of BNC polynomial ($p$) and some example $GA^2M$ polynomials ($q_1$, $q_2$, and $q_3$).}\label{table:polytermslinear}
\centering
\scriptsize
\renewcommand{\arraystretch}{2} % Wider row
\begin{tabular}{lccccccccccccccc}
\toprule
      & Const & $x_1$ & $x_2$ & $x_3$ & $x_4$ & $x_5$ & $x_6$ & $x_7$ & $x_8$ & $x_9$ & $x_{10}$ & $x_{11}$ & $x_{12}$ & $x_{13}$ & $x_{14}$ \\ 
 \midrule
$p$   &  0.11 &  1.87 & -2.53 & -1.71 & -2.82 & 2.08  & 1.56  & -1.17  & -1.93 & 3.96  & -1.94 & -0.35 & 2.87 &  0.10 & -1.20 \\
$q_1$ & -0.46 &  1.84 & -1.90 & -1.48 & -2.03 & 1.27  & 1.41  & -1.21 & -0.92 & 3.69  & -1.35 & -0.39 & 2.92 & -0.37 & -0.46 \\ 
$q_2$ &  0.45 &  2.18 & -2.97 & -2.36 & -2.86 & 2.19  & 1.14  & -0.88 & -0.85 & 3.61  & -0.99 & -0.50 & 2.36 &  0.35 & -0.50 \\
$q_3$ & -0.41 &  2.27 & -2.26 & -1.68 & -2.69 &  2.09 & 1.39  & -1.37 & -1.61 & 4.02  & -1.88 &  0.29 & 4.05 & -0.66 & -1.07 \\ 
\bottomrule
\end{tabular}
\end{table}

\begin{table}[htbp]\caption{Coefficients of \emph{quadratic} (interaction) terms of BNC polynomial ($p$) and some example $GA^2M$ polynomials ($q_1$, $q_2$, and $q_3$).}\label{table:polytermsquadratic}
\centering
\scriptsize
\setlength\tabcolsep{2.5pt} % Narrower columns
\renewcommand{\arraystretch}{2.2} % Wider rows
\scalebox{0.82}{
\begin{tabular}{ lcccccccccccccccccccc }\toprule
& \multicolumn{12}{c}{Terms present in $p$} & \multicolumn{6}{c}{Terms not present in $p$}\\
\cmidrule(lr){2-13} \cmidrule(lr){14-19} 
 &  $x_1x_2$ &  $x_1x_3$ &  $x_2x_4$ &  $x_2x_5$ &  $x_3x_6$ &  $x_3x_7$ &  $x_8x_9$ &  $x_8x_{10}$ &  $x_9x_{11}$ &  $x_9x_{12}$ &  $x_{10}x_{13}$ &  $x_{10}x_{14}$ & $x_{1}x_{6}$ & $x_{1}x_{7}$ & $x_{2}x_{3}$ & $x_{8}x_{12}$ &  $x_{8}x_{13}$ &  $x_{11}x_{12}$ \\ 
 \midrule
 $p$   & 3.03 & -0.24 & 3.80 & -1.98 & -2.24 & 1.55 & 0.62 & 4.31 & -2.27 & -3.31 & -0.54 & 0.95 &       &       &      &       &      &       \\
 $q_1$ & 1.88 &       & 2.57 & -1.18 & -1.88 & 1.06 &      & 3.39 & -1.48 & -3.07 &       &      &       &  0.44 &      & -0.51 & 0.67 & -0.39 \\
 $q_2$ & 3.45 &       & 3.39 & -1.71 & -1.75 & 1.65 &      & 3.66 & -1.79 & -2.97 & -1.03 &      & -0.33 & -0.05 & 0.45 &       &      &       \\
 $q_3$ & 2.80 & -0.86 & 3.51 & -2.07 & -2.38 & 1.79 &      & 4.14 & -2.00 & -3.57 &       & 0.94 &       &       &      &       & 0.21 & -1.23 \\
 \bottomrule
\end{tabular}
}
\end{table}

Based on Tables~\ref{table:polytermslinear} and \ref{table:polytermsquadratic}, we can make the following observations.

\begin{itemize}\itemsep0em
    \item The signs of the constant term do not always agree. The signs of the linear terms always agree, with the exceptions of $x_{11}$ and $x_{13}$.
    \item Of the 12 interaction terms present in $p$, 8 always appear in the $GA^2M$ polynomials. Their sign is always correct. The intersection sizes are 8, 9, 10, respectively.
    \item The coefficients of terms of $p$ missing or having incorrect sign are smaller than other coefficients.
    \item The accuracy of the affine parts are 0.7525 for $p$, and 0.7648, 0.7654, 0.7640
    for $q_1, q_2, q_3$, respectively, so the affine parts of $p$ are worse than those of $q_1, q_2, q_3$. This seems to be due to the fact that $GA^2M$, in a greedy manner, produces the affine terms first~\cite{ebm}.
    \item A comparison of the terms missing from $p$ with the new terms  added in $q_1, q_2, q_3$ shows that in several cases these can be paired up in such a way that a missing (parent, child) pair corresponds to a (grandparent, grandchild) or a (sibling1, sibling2) pair. For example,
    in $q_1$, the term $x_1 x_3$ is missing, but $x_1 x_7$ is added.
\end{itemize}

In summary, on this small example, using a simple evaluation, the $GA^2M$ polynomials seem to have good explainability properties. It seems to be of interest to have a theoretical understanding of these properties. 
Extensions to Bayesian networks for real-world applications (see, e.g.,~\cite{sesen}) could also be considered.

\section{A Complexity Lower Bound for Positive Polynomial Threshold Representations} \label{sec:mon}

In this section we give a separation result between the sizes of positive and general QTF representations of an explicitly defined monotone Boolean function.
 The weight of 
$x \in \{0, 1\}^n$ is $|x| = \sum_{i=1}^n x_i$. If $|x| = \ell$ then $|x|$ is on level $\ell$.
The Boolean threshold function $Th^n_{\ell}(x)$ has value 1 iff $|x| \ge \ell$.

 Let $n = 2 k$  and
\[ f_n(x)  = Th_{k + 1}(x) \vee \left(Th_k(x) \wedge \bigwedge_{j=1}^k (\bar{x}_{2j-1} \vee \bar{x}_{2j})\right). \]
A PTF representation of the function $f_4(x)$ is given in  (\ref{eq:ex}).

\begin{theorem} The function $f_n$

a) is a monotone Boolean function,

b) is a QTF of size $O(n)$ (and tree-width 1),

c) has a positive QTF representation of size $O(n^2)$,

d) every positive QTF representation of $f_n$ has size $\Omega(n^2)$ (and so tree-width $\Omega(n)$).
\end{theorem}
\begin{proof} For \emph{a)} note that $f_n$ is a \emph{slice function}, i.e., for some level $\ell$ it is 0 on level $\ell - 1$ and below, it is 1 on level $\ell + 1$ and above, and so the only non-constant level is level $\ell$. Every slice function is monotone. 

Part \emph{b)} follows by considering the polynomial
\[ p(x) = \sum_{i=1}^n x_i  - \left(\frac{1}{k} \sum_{j=1}^k x_{2j-1} x_{2j}\right) - k. \]
The hyperedges of the term hypergraph are the singletons and the matching 
$M = \{(2j - 1, 2j) \, : \, 1 \le j \le k\}$. 

Part \emph{c)} follows by considering 
\[ \sum_{i=1}^n x_i  \; + \; \frac{1}{
% {k \choose 2}
{\binom{k}{2}} % Using \choose causes a warning, so use \binom instead for native amsmath support
} \; \sum_{(i,j) \not\in M} x_i x_j \ge k + 1.\]

For Part \emph{d)} assume that $p(x) \ge t$ is a positive QTF representation of $f_n$.
The claim follows if we show that for every $i, j$,
where $1 \le i, j \le k$, the polynomial $p$ contains a mixed term $x_{2 i - a} x_{2 j - b}$ for some $a, b \in \{0, 1\}$ between the $i$'th and $j$'th blocks of odd and even bits.  Assume w.l.o.g. that $i = 1, j = 2$ and there is no mixed term between the first two blocks. Let $z = (0, 1, 0, 1, \ldots)$ be the alternating truth assignment to the variables $x_5, \ldots, x_n$. Then by definition
\[ p(0,0,1,1, z) < t, \,\,\, p(1,1,0,0, z) < t, \,\,\, p(1,0,1,0, z) \ge t, \,\,\,  p(0,1,0,1, z) \ge t.\]
Let $\gamma_{i,j}$ be the coefficient of $x_i x_j$. Then it holds that  $\gamma_{1,3} = \gamma_{2,4} = 0$. Thus
\[ p(0,0,1,1, z) + p(1,1,0,0, z) = p(1,0,1,0, z) + p(0,1,0,1, z) + \gamma_{1,2} + \gamma_{3,4},\]
a contradiction. 
\end{proof}

A first open question is to determine the size of higher degree positive PTF representations of $f_n$.
More generally, larger separations between positive and general PTF complexities would be of interest.

\subsection*{Acknowledgement}
We would like to thank Sebastian Bordt, Rich Caruana, P\'eter Hajnal, Lisa Hellerstein, Michal Moshkovitz and Lev Reyzin for discussions.
The third author is partially supported by NSF grants 2217023 and 2240532, and by the Artificial Intelligence National Laboratory Program (RRF-2.3.1-21-2022-00004).
Part of this work was done while visiting the Simons Institute for the Theory of Computing.

\bibliographystyle{plain}
\bibliography{final-Chubarian-Joyce-Turan}

\end{document}